%% file: anonymous-submission-latex-2026.tex
\documentclass[letterpaper]{article} 
\usepackage{aaai2026}  
\usepackage{times}  
\usepackage{helvet}  
\usepackage{courier}  
\usepackage[hyphens]{url}  
\usepackage{graphicx} 
\usepackage{natbib}  
\usepackage{caption} 
\usepackage{algorithm}
\usepackage{algorithmic}

\nocopyright 

\usepackage{newfloat}
\usepackage{listings}
\DeclareCaptionStyle{ruled}{labelfont=normalfont,labelsep=colon,strut=off} 
\floatstyle{ruled}
\newfloat{listing}{tb}{lst}{}
\floatname{listing}{Listing}
\title{Efficient Rule Induction by Ignoring Pointless Rules}
\author{
    Andrew Cropper\textsuperscript{\rm 1,2}\equalcontrib, 
    David M. Cerna\textsuperscript{\rm 3,4}\equalcontrib 
}
\affiliations{
    \textsuperscript{\rm 1}ELLIS Institute Finland\\
 \textsuperscript{\rm 2}University of Helsinki\\

 \textsuperscript{\rm 3}Dynatrace Research\\ \textsuperscript{\rm 4}Czech Academy of Sciences Institute of Computer Science\\
andrew.cropper@helsinki.fi, david.cerna@dynatrace.com, dcerna@cs.cas.cz

}

\usepackage{bibentry}

\input{preamble}

\begin{document}

\maketitle

\input{00-abstract}
\begin{links}
\link{Code}{https://github.com/logicand-learning-lab/aaai26-implications}
\end{links}

\input{01-introduction}
\input{02-related_work}
\input{03-problem_setting}

\input{04-algo}
\input{06-experiments}
\input{07-conclusion}
\input{08-Acknowledgements}

\bibliography{ijcai24}
\newpage

\input{08-Appendix}

\end{document}

%% file: preamble.tex
\usepackage[utf8]{inputenc}

\usepackage{xcolor}
\usepackage{booktabs}
\usepackage{amsthm}
\usepackage{listings}
\usepackage{inconsolata}
\usepackage{tikz}
\usepackage{pgfplots}
\usepackage{amsmath}
\usepackage{microtype}
\usepackage{algorithm}
\usepackage{amssymb}
\usepackage[titletoc]{appendix}
 \usepackage{cancel}
 \usepackage{multirow}
\usepackage{subfig}

\usepackage{pgfplots}
\usepgfplotslibrary{colorbrewer}
\pgfplotsset{compat = 1.15, cycle list/Set1-8} 
\usetikzlibrary{pgfplots.statistics, pgfplots.colorbrewer} 
\usepackage{pgfplotstable}

\newcommand{\name}{\textsc{Reducer}}

\newcommand{\popper}{\textsc{Popper}}

\newcommand{\ale}{\textsc{Aleph}}

\newcommand{\aspsynth}{\textsc{ASPSynth}}

\theoremstyle{definition}
\newtheorem{definition}{Definition}

\newtheorem{lemma}{Lemma}
\newtheorem{corollary}{Corollary}
\newtheorem{theorem}{Theorem}

\newtheorem{proposition}{Proposition}

\lstnewenvironment{myalgorithm}[1][] 
{
    \lstset{ 
        basicstyle=\normalfont\ttfamily,
        showspaces=false,               
        showstringspaces=false,         
        mathescape=true,
        numbers=left,
        escapeinside={*}{*},
        columns=flexible,
        numbersep=2pt,        
        keywordstyle=\bfseries,
        keywords={,and, return, not, def, in, if, elif, else, for, foreach, while, continue,break,}
        numbers=left,
        xleftmargin=1em,
        #1 
    }
}
{}

\usepackage{pgfkeys}
    \def\addlegendimage{\scriptsize\csname pgfplots@addlegendimage\endcsname}
\usepackage{comment}

%% file: 00-abstract.tex
\begin{abstract}
The goal of inductive logic programming (ILP) is to find a set of logical rules that generalises training examples and background knowledge.
We introduce an ILP approach that identifies \emph{pointless} rules.
A rule is pointless if it contains a redundant literal or cannot discriminate against negative examples. 
We show that ignoring pointless rules allows an ILP system to soundly prune the hypothesis space.
Our experiments on multiple domains, including visual reasoning and game playing, show that our approach can reduce learning times by 99\% whilst maintaining predictive accuracies.
\end{abstract}

%% file: 01-introduction.tex
\section{Introduction}

The goal of inductive logic programming (ILP) is to induce a hypothesis (a set of logical rules) that generalises training examples and background knowledge (BK) \cite{mugg:ilp,ilpintro}.

To illustrate ILP, suppose we have BK with the relations \emph{succ/2}, \emph{lt/2}, \emph{gt/1}, \emph{int/1}, \emph{even/1}, \emph{odd/1} and the following positive ($E^+$) and negative ($E^-$) examples:

\begin{center}
\begin{tabular}{l}
\emph{E$^+$ = \{f(5), f(7)\}}\\
\emph{E$^-$ = \{f(2), f(3), f(4), f(6), f(8), f(9)\}}\\
\end{tabular}
\end{center}

\noindent
Given this input, we might want to learn a rule such as:
\begin{center}
\begin{tabular}{l}
        \emph{$r_1$ = f(A) $\leftarrow$ odd(A), gt(A,3), lt(A,8)}  
\end{tabular}
\end{center}
\noindent
This rule says that \emph{f(A)} is true if $A$ is odd, greater than 3, and less than 8.

An ILP learner tests rules on the examples and BK and uses the outcome to guide the search.
For instance, suppose a learner tests the rule:
\begin{center}
\begin{tabular}{l}
        \emph{f(A) $\leftarrow$ even(A)}    
\end{tabular}
\end{center}
\noindent
This rule is too specific because it does not entail any positive example. 
Therefore, a learner can ignore its specialisations, such as:
\begin{center}
\begin{tabular}{l}
        \emph{f(A) $\leftarrow$ even(A), gt(A,2)}    
\end{tabular}
\end{center}
\noindent
Similarly, suppose a learner tests the rule:
\begin{center}
\begin{tabular}{l}
        \emph{$r_2$ = f(A) $\leftarrow$ odd(A), int(A)}
\end{tabular}
\end{center}
\noindent
This rule is too general because it entails two negative examples (\emph{f(3)} and \emph{f(9)}).
Therefore, a learner can ignore its generalisations.
A learner will not necessarily ignore its specialisations because one could still be helpful, such as:
\begin{center}
\begin{tabular}{l}
        \emph{$r_3$ = f(A) $\leftarrow$ odd(A), int(A), gt(A,3), lt(A,8)}
\end{tabular}
\end{center}
\noindent
This rule entails all the positive and none of the negative examples.

However, $r_2$ and $r_3$ are pointless because \emph{odd(A)} implies \emph{int(A)}, so we can remove \emph{int(A)} whilst preserving semantics.
We call such rules \emph{reducible} rules.

Now consider the rule:
\begin{center}
\begin{tabular}{l}
        \emph{f(A) $\leftarrow$ lt(A,B)}    
\end{tabular}
\end{center}
\noindent
This rule entails all the positive and all the negative examples, so a learner can ignore its generalisations.
However, although it entails all the negative examples, we cannot ignore its specialisations as one might still be a good rule, such as $r_1$.

Now consider the specialisation:
\begin{center}
\begin{tabular}{l}
        \emph{f(A) $\leftarrow$ lt(A,10)}    
\end{tabular}
\end{center}
\noindent
This rule entails all the positive and negative examples, so a learner can ignore its generalisations.
However, unlike the previous rule, a learner can ignore its specialisations.
The reason is that the literal \emph{lt(A,10)} implies all the negative examples, i.e. there is no negative example greater than 10.
Moreover, this literal cannot be further specialised (besides grounding the variable $A$ and thus disconnecting the head literal from the body). 
We call such rules \emph{indiscriminate} rules. 

The novelty of this paper is to show that ignoring reducible and indiscriminate rules allows us to efficiently and soundly prune the hypothesis space and thus improve learning performance.
Although existing approaches find reasons why a rule fails on examples, such as having erroneous literals \cite{mugg:metagold,musper} or rules  \cite{mis,prosynth}, they do not find reducible or indiscriminate rules.

We demonstrate our idea in the new ILP system \name{}.
This system builds on \popper{} \cite{popper,maxsynth}, which can learn optimal and recursive hypotheses from noisy data. 
The key novelty of \name{} is identifying pointless rules and building constraints from them to prune the hypothesis space.


\subsubsection*{Novelty and Contributions}
The novelty of this paper is \emph{identifying pointless rules and building constraints from them to prune the hypothesis space of an ILP system}.
The impact, which we demonstrate on multiple domains, is vastly reduced learning times.
Overall, we contribute the following:

\begin{itemize}
\item We define two types of pointless rules: \emph{reducible} and \emph{indiscriminate}.
We show that specialisations of reducible rules are also reducible (Proposition \ref{prop:redu_specs}).
Likewise, we show that specialisations of indiscriminate rules are also indiscriminate (Proposition \ref{prop:indi_specs}).
Finally, we show that a hypothesis with a reducible or indiscriminate rule is not optimal (Propositions \ref{prop:sound_sat3} and \ref{prop:sound_Indiscriminate}).

\item We introduce \name{}, an ILP system which identifies pointless rules in hypotheses, including recursive hypotheses, and builds constraints from them to prune the hypothesis space.
We prove that \name{} always learns an optimal hypothesis if one exists (Theorem \ref{thm:optcorrect}).

\item We experimentally show on multiple domains, including visual reasoning and game playing, that our approach can reduce learning times by 99\% whilst maintaining high predictive accuracies.
\end{itemize}

%% file: 02-related_work.tex
\section{Related Work}
\label{sec:related}

\textbf{ILP.}
Many ILP systems use bottom clauses \cite{progol,aleph} or variants \cite{xhail} to restrict the hypothesis space.
Building bottom clauses can be expensive \cite{progol}.
In the worst case, a learner needs a bottom clause for every positive example.
Bottom clause approaches struggle to learn hypotheses with recursion or predicate invention.
\name{} does not use bottom clauses.

\textbf{Redundancy in ILP.}
\citet{DBLP:conf/ilp/FonsecaCSC04} define \emph{self-redundant} clauses similar to reducible rules (Definition \ref{def:breducible}) but cannot guarantee that specialisations are redundant (unlike Proposition \ref{prop:redu_specs}) and require users to provide redundancy information.
\citet{RaedtR04} check redundancy before testing but require \textit{anti-monotonic} constraints (holding for generalisations, not specialisations) and do not identify implications between literals.
By contrast, \name{} finds implications to prune specialisations.
\citet{quickfoil} prune syntactic redundancy (e.g. duplicate variables), whereas we detect semantic redundancy.
\citet{DBLP:conf/ilp/SrinivasanK05} compress bottom clauses statistically; we avoid bottom clauses.

\textbf{Rule selection.}
Many systems \cite{aspal,hexmil,prosynth,aspsynth,ilasp4} precompute every possible rule in the hypothesis space and then search for a subset that generalises the examples.
Because they precompute all possible rules, they cannot learn rules with many literals and can build pointless rules. 
For instance, \textsc{ilasp4} \cite{ilasp4} will precompute all rules with \emph{int(A)} and \emph{even(A)} in the body.
By contrast, \name{} builds constraints from pointless rules to restrict rule generation. 
If \name{} sees a rule containing \emph{int(A)} and \emph{even(A)}, it identifies that \emph{even(A)} implies \emph{int(A)} and henceforth never builds a rule with both literals.


\textbf{Constraints.}
Many recent ILP systems frame the ILP problem as a constraint satisfaction problem \cite{aspal,atom,inspire,hexmil,aspsynth}.
\textsc{MUSPer} \cite{musper} finds minimal sub-hypotheses that can never be true and builds constraints from them to prune the hypothesis space.
For instance, given the rule 
\emph{f(A) $\leftarrow$ succ(A,B),odd(B),even(B)},
\textsc{MUSPer} can identify that a number cannot be both odd and even and then prunes the hypothesis space accordingly.
By contrast, \name{} finds pointless rules with redundancy.

\textbf{Rule induction.}
ILP approaches induce rules from data, similar to rule learning methods \cite{DBLP:conf/ruleml/FurnkranzK15} such as AMIE+ \cite{DBLP:journals/vldb/GalarragaTHS15} and RDFRules \cite{rdfrules}.
Most rule-mining methods are limited to unary and binary relations, require facts as input, and operate under an open-world assumption. 
By contrast, \name{} operates under a closed-world assumption, supports arbitrary-arity relations, and can learn from definite programs as background knowledge.

\textbf{Redundancy in AI.}
Theorem proving has preprocessing techniques for eliminating redundancies \cite{HoderKKV12,KhasidashviliK16,VukmirovicBH23}. 
In SAT, redundancy elimination techniques, such as blocked clause elimination~\cite{Kullmann99}, play an integral role in modern solvers~\cite{HeuleJB10a,BarnettCB20,BiereJK21}. 
Similar to this paper, SAT redundancy elimination identifies clauses containing literals that always resolve to tautologies.

%% file: 03-problem_setting.tex
\section{Problem Setting}
We assume familiarity with logic programming \cite{lloyd:book} but have included summaries in the supplementary material.
For clarity, we define some key terms. 
A rule $r$ is a definite clause of the form $h\leftarrow p_1,\ldots, p_n$ where $h, p_1,\ldots, p_n$ are literals, $head(r) = h$, and $body(r) = \{ p_1,\ldots, p_n\}$. 
A \emph{definite program} is a set of definite clauses. We use the term \emph{hypothesis} interchangeably with a definite program.
We denote the set of variables in a literal $l$ as $vars(l)$.
The variables of a rule $r$, denoted $vars(r)$, is defined as $vars(head(r)) \cup \bigcup_{p\in body(r)} vars(p)$.
Given a rule $r$ and a set of literals $C$, by $r\cup C$ we denote the rule $r'$ such that $head(r') = head(r)$ and $body(r') =  body(r)\cup C$.
A \emph{specialisation} of a rule $r$ is a rule $r'$ that is $\theta$-subsumptively more specific than $r$, where $\theta = \{v_1/t_1, …, v_n/t_n\}$ is a substitution that simultaneously replaces each variable $v_i$ by term $t_i$ (See~\cite{plotkin:thesis} and the Appendix).
We focus on a restricted form of $\theta$-subsumption which only considers whether the bodies of two rules with the same head literal are contained in one another, i.e., no variable renaming.
We formally define this restriction through a \emph{subrule} relation:


\begin{definition}[\textbf{Subrule}]
\label{def:satruleALT}
Let $r_1$ and $r_2$ be rules. 
Then $r_1$ is a \emph{subrule} of $r_2$, denoted $r_1\subseteq r_2$, if $head(r_1) = head(r_2)$ and $body(r_1)\subseteq body(r_2)$.
\end{definition}
\noindent
We generalise the subrule relation to a \emph{sub-hypothesis} relation. 
Unlike the subrule relation, the sub-hypothesis relation is not a restriction of $\theta$-subsumption:
\begin{definition}[\textbf{Sub-hypothesis}]
\label{def:subhypoth}
Let $h_1$ and $h_2$ be hypotheses and for all $r_1\in h_1$ there exists $r_2\in h_2$ such that $r_1\subseteq r_2$.
Then $h_1$ is a \emph{sub-hypothesis} of $h_2$, denoted $h_1\subseteq h_2$. 
\end{definition}
\noindent
The sub-hypothesis relation captures a particular type of hypothesis that we refer to as \emph{basic}. 
These are hypotheses for which specific rules do not occur as part of a recursive predicate definition:
\begin{definition}[\textbf{Basic}]
\label{def:basic}
Let $h$ be a hypothesis, 
$r_1$ a rule in $h$,
and for all $r_2$ in $h$, the head symbol of $r_1$ does not occur in a body literal of $r_2$.
Then \emph{$r_1$ is basic in $h$}.
\end{definition}
\noindent
As we show in Section~\ref{subsec:pointless}, under certain conditions we can prune hypotheses containing sub-hypotheses that are \emph{basic} with respect to a contained rule. 



\subsection{Inductive Logic Programming}
We formulate our approach in the ILP learning from entailment setting \cite{luc:book}.
We define an ILP input:

\begin{definition}[\textbf{ILP input}]
\label{def:probin}
An ILP input is a tuple $(E, B, \mathcal{H})$ where $E=(E^+,E^-)$ is a pair of sets of ground atoms denoting positive ($E^+$) and negative ($E^-$) examples, $B$ is background knowledge, and $\mathcal{H}$ is a hypothesis space, i.e., a set of possible hypotheses.
\end{definition}
\noindent
We restrict hypotheses and background knowledge to definite programs with the least Herbrand model semantics.

We define a cost function:
\begin{definition}[\textbf{Cost function}]
\label{def:cost_function}
Given an ILP input $(E, B, \mathcal{H})$, a cost function $cost_{E,B}~:~\mathcal{H}~\to~\mathbb{N}$ assigns a numerical cost to each hypothesis in $\mathcal{H}$.
\end{definition}

\noindent
 Given an ILP input and a cost function $cost_{E,B}$, we define an \emph{optimal} hypothesis:
\begin{definition}[\textbf{Optimal hypothesis}]
\label{def:opthyp}
Given an ILP input $(E, B, \mathcal{H})$ and a cost function \emph{cost$_{E,B}$}, a hypothesis $h \in \mathcal{H}$ is \emph{optimal} with respect to \emph{cost$_{E,B}$} when $\forall h' \in \mathcal{H}$, \emph{cost$_{E,B}$}($h$) $\leq$ \emph{cost$_{E,B}$}($h'$).
\end{definition}


\noindent
We use a cost function that first minimises misclassified training examples, then minimises the number of literals in a hypothesis.
False positives are negative examples entailed by $h \cup B$.
False negatives are positive examples not entailed by $h \cup B$.
We denote these as $fp_{E,B}(h)$ and $fn_{E,B}(h)$ respectively.
We define $size: \cal{H} \rightarrow {\mathbb{N}}$ as the number of literals in $h \in {\cal{H}}$. 
We use the cost function $cost_{E,B}(h) = (fp_{E,B}(h) + fn_{E,B}(h), size(h))$.


\subsection{Pointless Rules}
\label{subsec:pointless}
We want to find rules that cannot be in an optimal hypothesis. We focus on \emph{reducible} and \emph{indiscriminate} rules.



A \textit{reducible} rule contains a body literal that is implied by other body literals.
For example, consider the rules:
\begin{center}
\begin{tabular}{l}
\emph{r$_1$  = h $\leftarrow$ odd(A), int(A)}\\
\emph{r$_2$  = h $\leftarrow$ odd(A)}
\end{tabular}
\end{center}
\noindent
The rule \emph{r$_1$} is reducible because \emph{odd(A)}  implies \emph{int(A)}.
Therefore, \emph{r$_1$} is logically equivalent to \emph{r$_2$}.

As a second example, consider the rule: 
\begin{center}
\begin{tabular}{l}
\emph{h $\leftarrow$ gt(A,B), gt(B,C), gt(A,C)}\\
\end{tabular}
\end{center}
\noindent
This rule is reducible because the relation \emph{gt/2} is transitive, i.e. \emph{gt(A,B)} and \emph{gt(B,C)} imply \emph{gt(A,C)}.

Because it contains a redundant literal, a reducible rule cannot be in an optimal hypothesis.
However, a specialisation of a reducible rule could be in an optimal hypothesis.
For instance, consider the rule:
\begin{center}
\begin{tabular}{l}
\emph{$r_1$ = h $\leftarrow$ member(L,X), member(L,Y)}\\
\end{tabular}
\end{center}
\noindent
In this rule, \emph{member(L,X)} implies \emph{member(L,Y)} and vice-versa, so one of the literals is redundant.
However, we could still specialise this rule as:
\begin{center}
\begin{tabular}{l}
\emph{$r_2$ = h $\leftarrow$ member(L,X), member(L,Y), gt(X,Y)}\\
\end{tabular}
\end{center}


\noindent
Rules $r_1$ and $r_2$ are not logically equivalent, and $r_2$ could be in an optimal hypothesis.

A key contribution of this paper is to identify reducible rules where we can prune all their specialisations.
The idea is to identify a redundant \emph{captured} literal.
A captured literal is one where all of its variables appear elsewhere in the rule.
For instance, consider the rule:
\begin{center}
\begin{tabular}{l}
\emph{h $\leftarrow$ succ(A,B), succ(B,C), gt(C,A), gt(C,D)}\\
\end{tabular}
\end{center}
\noindent In this rule, the literal \emph{gt(C,A)} is captured because all its variables appear elsewhere in the rule.
By contrast, the literal \emph{gt(C,D)} is not captured because the variable \emph{D} does not appear elsewhere in the rule.

We define a captured literal:
\begin{definition}[\textbf{Captured literal}]
\label{def:capturedLiteral}
Let $r$ be a rule,
$l \in body(r)$, 
and
$vars(l) \subseteq vars( body(r) \setminus \{l\}) \cup vars(head(r))$. 
Then $l$ is $r$-\emph{captured}. 
\end{definition}

\noindent
If a literal is captured in a rule then it is captured in its specialisations:
\begin{lemma}
\label{prop:capTrans}
Let $r_1$ and $r_2$ be rules such that 
$r_2\subseteq r_1$,
$l\in body(r_2)$,
and $l$ be $r_2$-captured.
Then $l$ is $r_1$-captured.
\end{lemma}
\begin{proof}
Follows from Definition~\ref{def:satruleALT} as the subrule relation preserves variable occurrence. 
\end{proof}

\noindent
We define a reducible rule:


\begin{definition}[\textbf{Reducible}]
\label{def:breducible}
Let $r$ be a rule, 
$B$ be BK,
$l \in body(r)$ be $r$-captured,
and $B\models (body(r) \setminus \{l\})\rightarrow l$.
Then  $r$ is \emph{reducible}.
\end{definition}

\noindent Some specialisations of a reducible rule are reducible:

\begin{proposition}[\textbf{Reducible specialisations}] \label{prop:redu_specs}
Let $B$ be BK, 
$r_1$ be a reducible rule, 
and $r_1\subseteq r_2$. 
Then $r_2$ is reducible.
\end{proposition}

\begin{proof}
 Let $l$ be a $r_1$-captured literal and $B\models (body(r_1) \setminus \{l\})\rightarrow l$. 
 By Lemma~\ref{prop:capTrans}, $l$ is also $r_2$-captured. 
 Let $C= body(r_2)\setminus body(r_1)$. 
Then, $(body(r_1) \setminus \{l\})\rightarrow l$ is subsumptively more general than $(body(r_1\cup C) \setminus \{l\})\rightarrow l$, where $r_2=r_1\cup C$, i.e., if the former holds for $r_1$ than the latter holds for $r_2$. Thus, $r_2$ is reducible.
\end{proof}


\noindent 
Certain hypotheses that contain a sub-hypothesis with reducible rules are not optimal:


\begin{proposition}[\textbf{Reducible soundness}] \label{prop:sound_sat3}
Let $B$ be BK,
$h_1$ be a hypothesis, 
$h_2\subseteq h_1$, 
$r_1$ be basic rule in $h_1$, 
$r_2\in h_2$,
$r_2\subseteq r_1$,
and $r_2$ be reducible with respect to $B$.
Then $h_1$ is not optimal. 
\end{proposition}

\begin{proof} By Proposition~\ref{prop:redu_specs}, $r_1$ is also reducible implying that there exists an $l\in body(r_1)$ and rule $r_3\subset r_1$ such that (i)  $B\models (body(r_1)\setminus \{l\} )\rightarrow l$, (ii) $r_1=r_3\cup\{l\}$, and (iii) $|r_3|< |r_1|$. Furthermore, $r_3$ is basic given that $r_1$ is basic. Let $h_3 = (h_1\setminus \{r_1\})\cup\{r_3\}$. Then $cost_{E,B}(h_3) < cost_{E,B}(h_1)$, i.e. $h_1$ is not optimal.
\end{proof}


\noindent
This proposition implies that if we find a reducible rule, we can ignore hypotheses that include this rule or its specialisations.

We introduce \emph{indiscriminate} rules, a weakening of reducible rules. To motivate them, consider the rule:
\begin{center}
\begin{tabular}{l}
\emph{f(A) $\leftarrow$ odd(A), lt(A,10)}
\end{tabular}
\end{center}

\noindent
This rule is not reducible because \emph{odd(A)} does not imply 
\emph{lt(A,10)}, nor does \emph{lt(A,10)} imply \emph{odd(A)}.
However, suppose we have the negative examples $E^- = \{f(1), f(2), f(3)\}$.
For these examples, the literal \emph{lt(A,10)} implies all the negative examples.
In other words, there is no negative example greater than 10.
Therefore, this literal (and thus this rule) is pointless because it cannot discriminate against the negative examples.
We formalise this notion of an indiscriminate rule:
\begin{definition}[\textbf{Indiscriminate}]
\label{def:indiscriminate}
Let $r$ be a rule, 
$B$ be BK,
$E^-$ be negative examples with the same predicate symbol as the head of $r$,
$l \in body(r)$ be $r$-captured,
and for all $e\in E^-$, $B\models (body(r)\setminus \{l\})\theta_e \rightarrow l\theta_e^r$ where $\theta_e^r$ is a substitution with domain $vars(head(r))$ such that $head(r)\theta_e^r = e$.
Then  $r$ is \emph{indiscriminate}.
\end{definition}
\noindent Under certain conditions, specialisations of an indiscriminate rule are indiscriminate:

\begin{proposition}[\textbf{Indiscriminate specialisations}] \label{prop:indi_specs}
Let $B$ be BK, 
$r_1$ be an indiscriminate rule, 
$r_1\subseteq r_2$, 
and  $E^-$ be negative examples with the same predicate symbol as the head of $r_1$. 
Then $r_2$ is indiscriminate.
\end{proposition}

\begin{proof}
 Let $l$ be an $r_1$-captured literal and for all $e\in E^-$, $B\models (body(r_1)\setminus \{l\})\theta_e^{r_1} \rightarrow l\theta_e^{r_1}$. 
 By Lemma~\ref{prop:capTrans}, $l$ is also $r_2$-captured. 
 Let $C= body(r_2)\setminus body(r_1)$. 
 We deduce the following for all $e\in E^-$: $(body(r_1)\setminus \{l\})\theta_e^{r_1} \rightarrow l\theta_e^{r_1}$ is subsumptively more general than $(body(r_1\cup C)\setminus \{l\})\theta_e^{r_2} \rightarrow l\theta_e^{r_2}$ where $r_2 = (r_1\cup C)$, i.e., if the former holds for $r_1$, than the latter holds for $r_2$. Thus, $r_2$ is indiscriminate.
\end{proof}


\noindent
As with reducible rules, some hypotheses with an indiscriminate rule are not optimal:

\begin{proposition}[\textbf{Indiscriminate soundness}] \label{prop:sound_Indiscriminate}
Let $B$ be BK,
$E^-$ be negative examples,
$h_1$ be a hypothesis, 
$r_1$ be a basic rule in $h_1$, 
$h_2\subseteq h_1$, 
$r_2\in h_2$,
$r_2\subseteq r_1$,
and $r_2$ be indiscriminate with respect to $B$ and $E^-$.
Then $h_1$ is not optimal. 

\end{proposition}

\begin{proof}
By Proposition~\ref{prop:indi_specs}, $r_1$ is also indiscriminate implying that there exists  $l\in body(r_1)$ and $r_3\subset r_1$  such that  (i) for all $e\in E^-$, $B\models (body(r_1)\setminus \{l\})\theta_e^{r_1} \rightarrow l\theta_e^{r_1}$, (ii) $r_1=r_3\cup \{l\}$, and (iii) $|r_3|< |r_1|$. Furthermore, $r_3$ is basic given that $r_1$ is basic. Let $h_3 = (h_1\setminus \{r_1\})\cup\{r_3\}$.
Then $cost_{E,B}(h_3) < cost_{E,B}(h_1)$, i.e. $h_1$ is not optimal.
\end{proof}

\noindent 
A hypothesis  with a reducible or indiscriminate rule is not optimal:
\begin{definition}[\textbf{Pointless}]
\label{def:pointless}
Let $(E, B, \mathcal{H})$ be ILP input. 
A hypothesis $h\in \mathcal{H}$ is \emph{pointless} if there exists $r\in h$ such that $r$ is reducible with respect to $B$ or indiscriminate with respect to $B$ and $E^-$.
\end{definition}

\begin{corollary}
\label{col:1}
A pointless hypothesis is not optimal.
\end{corollary}


\noindent
In Section \ref{sec:algo}, we introduce \name{}, which identifies pointless rules and prunes them from the hypothesis space.

%% file: 04-algo.tex
\section{Algorithm}
\label{sec:algo}

We now describe \name{} (Algorithm \ref{alg:explainer}), which builds on \popper{} by identifying pointless rules and using them to prune the hypothesis space.
\name{} takes as input background knowledge (\emph{bk}), positive (\emph{pos}) and negative (\emph{neg}) training examples, and a maximum hypothesis size (\emph{max\_size}).
\name{} follows \popper{} and uses a generate, test, and constrain loop to find an optimal hypothesis (Definition \ref{def:opthyp}).
\name{} starts with an answer set program $\mathcal{P}$.
The (stable) models of $\mathcal{P}$ correspond to hypotheses (definite programs) and represent the hypothesis space.
In other words, $\mathcal{P}$ encodes in ASP the search for a syntactically valid hypothesis that satisfies given constraints.
In the generate stage (line 6), \name{} uses an answer set programming (ASP) solver to find a model of $\mathcal{P}$.
If there is no model, \name{} increments the hypothesis size and loops again (lines 7-9).
If there is a model, \name{} converts it to a hypothesis $h$.
In the test stage, it uses Prolog to test $h$ on the training examples (line 10).
If $h$ is better (according to a cost function) than the previous best-seen hypothesis, \name{} sets the best hypothesis to $h$ (line 12).
In the constrain stage, \name{} builds hypothesis constraints (represented as ASP constraints) from $h$.
This step is identical to \popper{}.
For instance, if $h$ does not entail any positive examples, \name{} builds a specialisation constraint to prune its specialisations.
\name{} saves these constraints (line 13) which it passes to the generate stage to add to $\mathcal{P}$ to prune models and thus prune the hypothesis space.

\begin{algorithm}[ht!]
\footnotesize
{
\begin{myalgorithm}[]
def reducer(bk, pos, neg, max_size):
  cons = {}
  size = 1
  best_h, best_score = none, $\infty$
  while size $\leq$ max_size:
    h = generate(cons, size)
    if h == UNSAT:
      size += 1
      continue
    h_score = test(bk, pos, neg, h)
    if h_score < best_score:
        best_h, best_score = h, h_score
    cons += build_cons(h, h_score)
    if pointless(h, neg, bk):
      cons += build_spec_basic_con(h)
      cons += build_gen_basic_con(h)
  return best_h
\end{myalgorithm}
\caption{
\name{}.
}
\label{alg:explainer}
}
\end{algorithm}

The novelty of \name{} is checking whether a hypothesis has a pointless rule (line 14) by calling Algorithm \ref{alg:pointless}.
We describe Algorithm \ref{alg:pointless} below.
If a hypothesis has a pointless rule, \name{} builds constraints to prune its generalisations and specialisations, in which the pointless rule is basic.
\name{} passes these constraints to the generate stage to add to $\mathcal{P}$ to prune models and thus prune the hypothesis space of other hypotheses with pointless rules.

\name{} continues its loop until it exhausts the models of $\mathcal{P}$ or reaches a timeout, at which point it returns the best-seen hypothesis.

\begin{algorithm}[ht!]
\footnotesize
{
\begin{myalgorithm}[]  
def pointless(h, neg, bk):
  for rule in h:
    if not basic(rule, h):
      continue
    head, body = rule
    for literal in body:
      $\mbox{body}'$ = body-literal
      if not captured(head, $\mbox{body}'$, literal):
        continue       
      if reducible(bk, $\mbox{body}'$, literal):
        return true
      if indiscriminate(bk, neg, rule, head, $\mbox{body}'$):
        return true
  return false

def reducible(bk, neg,  $\mbox{body}'$, literal):
  $\mbox{rule}'$ = ($\bot$, $\mbox{body}'$ $\cup$ $\{ \neg 
 \mbox{literal}\}$)
  return unsat(bk, $\mbox{rule}'$)
  
def indiscriminate(bk, neg, rule, head, $\mbox{body}'$):
  $\mbox{rule}'$ = (head, $\mbox{body}'$)
  s1 = neg_covered(bk, neg, rule)
  s2 = neg_covered(bk, neg, $\mbox{rule}'$)
  return s1 == s2
\end{myalgorithm}
\caption{
Finding pointless rules.
}
\label{alg:pointless}
}
\end{algorithm}

\subsection*{Pointless Rules}
Algorithm~\ref{alg:pointless} checks whether a hypothesis has a rule which 
(i) is basic\footnote{
We enforce other syntactic restrictions, such as forcing literals in a rule to be connected, where they cannot be partitioned into two sets such that the variables in the literals of one set are disjoint from the variables in the literals of the other set.}  (Definition \ref{def:basic}), 
(ii) has a captured literal (Definition \ref{def:capturedLiteral}, line 8), 
and (iii) is reducible (Definition~\ref{def:breducible}, line 10) or indiscriminate (Definition~\ref{def:indiscriminate}, line 12).  
Algorithm \ref{alg:pointless} includes two subprocedures for checking if a rule is reducible or indiscriminate. 
The first procedure (lines 16-18) checks if the query $\mbox{body}'\cup \{ \neg 
 \mbox{literal}\}$  is unsatisfiable over the background knowledge. 
 This query determines reducibility because it checks whether a literal is implied by the body. 
 The second subprocedure (lines 20-24) checks if the subrule without the captured literal entails the same negative examples as the original rule. 
 This query matches Definition~\ref{def:indiscriminate}.
We use Prolog to perform the unsat (line 18) and coverage (lines 22-23) checks.
\subsubsection{Correctness}

We show that \name{} is correct:

\begin{theorem}[\textbf{\name{} correctness}]
\label{thm:optcorrect}
\name{} returns an optimal hypothesis if one exists.
\end{theorem}
\begin{proof}
\citet{popper} show that given optimally sound constraints (an optimally sound constraint never prunes an optimal hypothesis), \popper{} returns an optimal hypothesis if one exists (Theorem 1).
\name{} builds on \popper{} by pruning pointless (one with a reducible or indiscriminate rule) hypotheses.
By Corollary \ref{col:1}, a pointless hypothesis is not optimal.
Therefore, \name{} never prunes an optimal hypothesis so returns one if it exists.
\end{proof}

%% file: 06-experiments.tex
\section{Experiments}
We claim that pruning pointless rules enables sound and efficient pruning of the hypothesis space, thereby improving learning performance.
To test this claim, our experiments aim to answer the question:
\begin{enumerate}
\item[\textbf{Q1}] Can pruning pointless rules reduce learning times whilst maintaining predictive accuracies?
\end{enumerate}

\noindent
To answer \textbf{Q1}, we compare the performance of \name{} against \popper{}.
As \name{} builds on \popper{}, the only experimental difference between the systems is the ability to identify pointless rules and build constraints from them to prune the hypothesis space.
Therefore, this comparison directly tests our claim.

Checking whether a rule is reducible or indiscriminate requires testing hypotheses on the BK and examples and thus incurs an overhead cost.
To understand the cost of pruning pointless rules, our experiments aim to answer the question:

\begin{enumerate}
\item[\textbf{Q2}] What is the overhead of pruning pointless rules?
\end{enumerate}

\noindent
\textbf{Q1} explores whether the combination of both types of pointless rules, reducible (Definition \ref{def:breducible}) and indiscriminate (Definition \ref{def:indiscriminate}), can improve learning performance.
To understand the impact of each type of rule, our experiments aim to answer the question:

\begin{enumerate}
\item[\textbf{Q3}] Can pruning reducible or indiscriminate rules alone reduce learning times?
\end{enumerate}
To answer \textbf{Q3}, we compare the performance of \name{} against \popper{} but where \name{} identifies only reducible rules or only indiscriminate rules, but not both.

Comparing \name{} against other systems besides \popper{} will not allow us to evaluate the idea of pruning pointless rules because it will not allow us to identify the source of empirical gains.
However, many people expect comparisons against other systems.
Therefore, we ask the question:
\begin{description}
\item[Q4] How does \name{} compare to other approaches?
\end{description}
\noindent
To answer \textbf{Q4}, we compare \name{} against \popper{}, \ale{} \cite{aleph}, and \aspsynth{} \cite{aspsynth}.


\subsubsection{Setup}
Each task contains training and testing examples and background knowledge.
We use the training examples to train the ILP system to learn a hypothesis.
We test a hypothesis on the testing examples.
Given a hypothesis $h$, background knowledge $B$, and a set of examples, 
a \emph{true positive} (tp) is a positive example entailed by $h \cup B$, 
a \emph{true negative} (tn) is a negative example not entailed by $h \cup B$, 
a \emph{false positive} (fp) is a negative example entailed by $h \cup B$, 
and a \emph{false negative} (fn) is a positive example not entailed by $h \cup B$.
We measure predictive accuracy as \emph{balanced accuracy}:
$\frac{1}{2} \left( \frac{tp(h)}{tp(h) + fn(h)} + \frac{tn(h)}{tn(h) + fp(h)} \right)$.
In all experiments, we measure predictive accuracy, termination times, and the time taken to discover pointless rules, i.e. the overhead of our approach. 
    For \textbf{Q1} and \textbf{Q2} we use a timeout of 60 minutes per task.
For \textbf{Q3}, we use a timeout of 10 minutes per task.
We repeat each experiment 10 times.
We plot and report 95\% confidence intervals (CI).
We compute 95\% CI via bootstrapping when data is non-normal.
To determine statistical significance, we apply either a paired t-test or the Wilcoxon signed-rank test, depending on whether the differences are normally distributed. 
We use the Benjamini–Hochberg procedure to correct for multiple comparisons.
All the systems use similar biases.
However, as with other rule selection approaches (Section \ref{sec:related}), \aspsynth{} precomputes all possible rules of a certain size and then uses an ASP solver to find a subset.
It is infeasible to precompute all possible rules, so we set the maximum rule size to 4.
We use an AWS m6a.16xlarge instance to run experiments where each learning task uses a single core.

\paragraph{Reproducibility.}
The code and experimental data for reproducing the experiments are available as supplementary material and will be made publicly available if the paper is accepted for publication.    

\subsubsection{Domains}
We use 449 tasks from several domains:

\textbf{1D-ARC.} This dataset \cite{onedarc} contains visual reasoning tasks inspired by the abstract reasoning corpus \cite{arc}.


\textbf{IGGP.} In inductive general game playing (IGGP) \cite{iggp}, the task is to induce rules from game traces from the general game playing competition \cite{ggp}.

\textbf{IMDB.}
We use a real-world dataset which contains relations between movies, actors, and directors \cite{imdb}. 

\textbf{List functions.} 
The goal of each task in this dataset is to identify a function that maps input lists to output lists, where list elements are natural numbers \cite{ruleefficient}.

\textbf{Trains.}
The goal is to find a hypothesis that distinguishes east and west trains \cite{michalski:trains}.

\textbf{Zendo.}
Zendo is a multiplayer game where players must discover a secret rule by building structures.


\subsection{Results}

\subsubsection{Q1. Can pruning Pointless Rules Reduce Learning Times Whilst Maintaining Predictive Accuracies?}
Figure \ref{fig:q1_times} shows the difference in learning times of \name{} vs \popper{}.
In other words, Figure \ref{fig:q1_times} shows the reduction in learning times by ignoring pointless rules.
Figure \ref{fig:q1_times} shows that \name{} consistently and drastically reduces learning times.
Significance tests confirm $(p < 0.05)$ that \name{} reduces learning times on 96/449 (21\%) tasks and increases learning times on 4/449 (1\%) tasks.
There is no significant difference in the other tasks.
The mean decrease in learning time is $28 \pm 4$ minutes and the median is 24 minutes with 95\% CI between 11 and 43 minutes.
The mean increase in learning time is $8 \pm 16$ minutes and the median is 0 minutes with 95\% CI between 0 and 34 minutes.
These are minimum improvements because \popper{} often times out after 60 minutes.
With a longer timeout, we would likely see greater improvements.

\begin{figure}[h!]
\centering
    \includegraphics[scale=1.1]{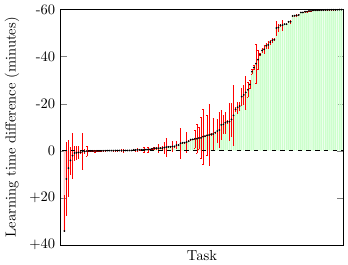}    
\caption{Learning time improvement of \name{} over \popper{}.
Values above the zero line indicate that \name{} reduces learning times.
The tasks are ordered by the learning time improvement.
For legibility, we only show tasks where the learning times differ by more than 1 second.
}
\label{fig:q1_times}
\end{figure}

Figure \ref{fig:q1_accs} shows the difference in predictive accuracy of \name{} vs \popper{}.
Significance tests confirm $(p < 0.05)$ that \name{} increases accuracy on 6/449 (1\%) tasks.
It decreases accuracy on 5/449 (1\%) tasks.
There is no significant difference in the other tasks.
The slight accuracy discrepancy is because multiple optimal hypotheses may exist. 
The two systems may find different ones. Two hypotheses with the same training cost might have different test accuracy. 


\begin{figure}[h!]
\centering
\includegraphics[scale=0.95]{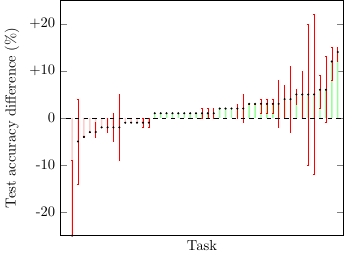}
\caption{Predictive accuracy improvement of \name{} compared to \popper{}, i.e. when ignoring pointless rules.
}
\label{fig:q1_accs}
\end{figure}








\name{} can drastically improve learning performance.
For instance, in the IGGP dataset, one of the tasks is to learn a set of rules to describe a legal move in the \emph{eight puzzle} game.
For this task, both \name{} and \popper{} learn hypotheses with 100\% accuracy.
However, whereas \popper{} does not terminate (prove optimality) within 60 minutes, \name{} terminates (proves optimality) after only $12 \pm 0$ seconds, a 99\% improvement.
A reducible rule that \name{} finds is:
\begin{center}
\begin{tabular}{l}
\emph{legal\_move(A,B,C,D) $\leftarrow$ succ(D,E), pos1(D), pos2(E)}
\end{tabular}
\end{center}

\noindent
This rule is reducible because the \emph{pos$_i$} relations denote positions in the game board, where 
\emph{pos$_i$} precedes \emph{pos$_{i+1}$}.
Therefore, if $pos_1(D)$ is true and $E$ is the successor of $D$ then $pos_2(E)$ must be true, and this literal is therefore redundant.

Three indiscriminate rules that \name{} finds are:
\begin{center}
\begin{tabular}{l}
\emph{legal\_move(A,B,C,D) $\leftarrow$ role(B)}\\
\emph{legal\_move(A,B,C,D) $\leftarrow$ index(C)}\\
\emph{legal\_move(A,B,C,D) $\leftarrow$ index(D)}\\
\end{tabular}
\end{center}

\noindent
These rules are indiscriminate because \emph{role(B)}, \emph{index(C)}, and \emph{index(D)} are true for every negative example and thus are redundant in the rule.


Overall, these results suggest that the answer to \textbf{Q1} is yes, pruning pointless can drastically improve learning times whilst maintaining high predictive accuracies.

\subsubsection{Q2. What Is the Overhead of Finding Pointless Rules?}

Figure \ref{fig:overhead} shows the ratio of learning time spent finding pointless rules.
The mean overhead is 2\%. 
The maximum is 80\%.
The overhead is less than 10\% on 85\% of the tasks.
Overall, these results suggest that the answer to \textbf{Q2} is that the overhead of pruning pointless is typically small ($<$10\%).

\begin{figure}[h!]
\centering
\includegraphics[scale=1.16]{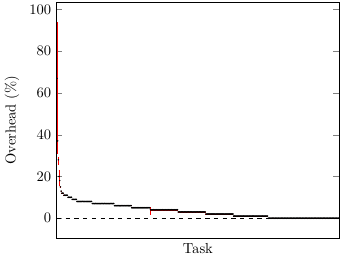}    
\caption{Overhead of finding pointless rules.}
\label{fig:overhead}
\end{figure}

\subsubsection{Q3. Can pruning Reducible or Indiscriminate Rules Alone Reduce Learning Times?}
Figure \ref{fig:q3} shows how the learning change depending on whether \name{} prunes only reducible or only indiscriminate rules.
Significance tests confirm $(p < 0.05)$ that pruning only reducible rules reduces learning times on 93/449 (21\%) tasks and increases learning times on 1/449 (0\%) tasks.
There is no significant difference in the other tasks.
The mean decrease in learning time is $26 \pm 5$ minutes and the median is 17 minutes with 95\% CI between 7 and 34 minutes.
The mean and median increase is 1 minute. 
Pruning only indiscriminate rules reduces learning times on 101/449 (22\%) tasks and increases learning times on 5/449 (1\%) tasks.
There is no significant difference in the other tasks.
The mean decrease in learning time is $28 \pm 4$ minutes and the median is 22 minutes with 95\% CI between 8 and 42 minutes.
The mean increase is $6 \pm 11$ minutes and the median increase is 0 minutes with 95\% CI between 0 and 28 minutes.
Overall, these results suggest that the answer to \textbf{Q3} is that both types of pointless rules drastically reduce learning times.

\begin{figure}[h!]
\centering
    \includegraphics[scale=0.95]{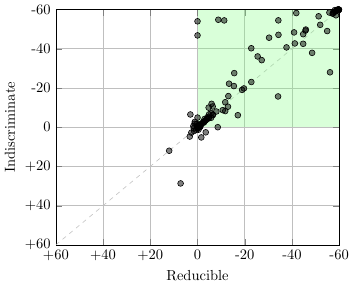}    
\caption{
Learning time improvement (minutes) when pruning reducible and indiscriminate rules.
Values in the upper right quadrant indicate that pruning either reducible rules or indiscriminate rules alone reduces learning time.
}
\label{fig:q3}
\end{figure}

\subsubsection{Q4. How Does \name{} Compare to Other Approaches?}
Table \ref{tab:agg_acc} shows the predictive accuracies aggregated per domain of all the systems.
\name{} has higher accuracy than \ale{} and \aspsynth{} on every domain.
\name{} has higher or equal accuracy than \popper{} on every domain.
Overall, these results suggest that the answer to \textbf{Q4} is that \name{} compares favourably to existing approaches in terms of predictive accuracy.

\begin{table}[ht]
\small
\centering
\begin{tabular}{@{}l|cccc@{}}
\textbf{Task} & \textbf{\ale{}} & \textbf{\aspsynth{}} & \textbf{\popper{}} & \textbf{\name{}}\\
\midrule
\emph{1d} & 54 $\pm$ 4 & 89 $\pm$ 3 & 92 $\pm$ 1 & 92 $\pm$ 1\\
\emph{alzheimer} & 51 $\pm$ 1 & 57 $\pm$ 2 & 75 $\pm$ 1 & 75 $\pm$ 1\\
\emph{iggp} & 57 $\pm$ 5 & 68 $\pm$ 0 & 85 $\pm$ 0 & 85 $\pm$ 0\\
\emph{imdb} & 50 $\pm$ 0 & 98 $\pm$ 0 & 100 $\pm$ 0 & 100 $\pm$ 0\\
\emph{jr} & 64 $\pm$ 11 & 83 $\pm$ 2 & 96 $\pm$ 0 & 97 $\pm$ 0\\
\emph{trains} & 50 $\pm$ 0 & 76 $\pm$ 2 & 87 $\pm$ 9 & 94 $\pm$ 7\\
\emph{zendo} & 53 $\pm$ 6 & 82 $\pm$ 0 & 84 $\pm$ 3 & 84 $\pm$ 3\\
\end{tabular}
\caption{
Aggregated predictive accuracies (\%).
}
\label{tab:agg_acc}
\end{table}


%% file: 07-conclusion.tex
\section{Conclusions and Limitations}

We have introduced an approach that identifies and ignores pointless (reducible and indiscriminate) rules.
We have shown that ignoring pointless rules is optimally sound (Propositions \ref{prop:sound_sat3} and \ref{prop:sound_Indiscriminate}).
We implemented our approach in \name{}, which identifies pointless rules in hypotheses and builds constraints from them to prune the hypothesis space.
We have proven that \name{} always learns an optimal hypothesis if one exists (Theorem \ref{thm:optcorrect}).
We have experimentally shown on multiple domains, including visual reasoning and game playing, that our approach can reduce learning times by 99\% whilst maintaining high predictive accuracies.



\paragraph{Limitations.}
The benefits of ignoring pointless rules should generalise to other ILP approaches.
For instance, the improvements should directly improve \textsc{hopper} \cite{hopper}, which learns higher-order programs, and \textsc{propper} \cite{propper}, which uses neurosymbolic inference to learn programs from probabilistic data.
Future work should empirically show how our idea benefits these systems.


%% file: 08-Acknowledgements.tex

\section*{Acknowledgments}
Andrew Cropper was supported by his EPSRC fellowship (EP/V040340/1). David M. Cerna was supported by the
Czech Science Foundation Grant 22-06414L and Cost Action
CA20111 EuroProofNet. 

%% file: 08-Appendix.tex
\section{Appendix}
\label{sec:terminology}
\subsection{Logic Programming}
We assume familiarity with logic programming \cite{lloyd:book} but restate some key relevant notation. A \emph{variable} is a string of characters starting with an uppercase letter. A \emph{predicate} symbol is a string of characters starting with a lowercase letter. The \emph{arity} $n$ of a function or predicate symbol is the number of arguments it takes. An \emph{atom} is a tuple $p(t_1, ..., t_n)$, where $p$ is a predicate of arity $n$ and $t_1$, ..., $t_n$ are terms, either variables or constants. An atom is \emph{ground} if it contains no variables. A \emph{literal} is an atom or the negation of an atom. A \emph{clause} is a set of literals.
A \emph{clausal theory} is a set of clauses. A \emph{constraint} is a clause without a non-negated literal. A \emph{definite} rule is a clause with exactly one non-negated literal. A \emph{program} is a set of definite rules. A \emph{substitution} $\theta = \{v_1 / t_1, ..., v_n/t_n \}$ is the simultaneous replacement of each variable $v_i$ by its corresponding term $t_i$. 
A rule $c_1$ \emph{subsumes} a rule $c_2$ if and only if there exists a substitution $\theta$ such that $c_1 \theta \subseteq c_2$. 
A program $h_1$ subsumes a program $h_2$, denoted $h_1 \preceq h_2$, if and only if $\forall c_2 \in h_2, \exists c_1 \in h_1$ such that $c_1$ subsumes $c_2$. A program $h_1$ is a \emph{specialisation} of a program $h_2$ if and only if $h_2 \preceq h_1$. A program $h_1$ is a \emph{generalisation} of a program $h_2$ if and only if $h_1 \preceq h_2$.
\subsection{Answer Set Programming}
We also assume familiarity with answer set programming \cite{asp} but restate some key relevant notation \cite{ilasp}.
A \emph{literal} can be either an atom $p$ or its \emph{default negation} $\text{not } p$ (often called \emph{negation by failure}). A normal rule is of the form $h \leftarrow b_1, ..., b_n, \text{not } c_1,... \text{not } c_m$. where $h$ is the \emph{head} of the rule, $b_1, ..., b_n, \text{not } c_1,... \text{not } c_m$ (collectively) is the \emph{body} of the rule, and all $h$, $b_i$, and $c_j$ are atoms. A \emph{constraint} is of the form $\leftarrow b_1, ..., b_n, \text{not } c_1,... \text{not } c_m.$ where the empty head means false. A \emph{choice rule} is an expression of the form $l\{h_1,...,h_m\}u \leftarrow b_1,...,b_n, \text{not } c_1,... \text{not } c_m$ where the head $l\{h_1,...,h_m\}u$ is called an \emph{aggregate}. In an aggregate, $l$ and $u$ are integers and $h_i$, for $1 \leq i \leq m$, are atoms. An \emph{answer set program} $P$ is a finite set of normal rules, constraints, and choice rules. Given an answer set program $P$, the \emph{Herbrand base} of $P$, denoted
as ${HB}_P$, is the set of all ground (variable free) atoms that can be formed from the predicates and constants that appear in $P$. When $P$ includes only normal rules, a set $A \in {HB}_P$ is an \emph{answer set} of $P$ iff it is the minimal model of the  \emph{reduct} $P^A$, which is the program constructed from the grounding of $P$ by first removing any rule whose body contains a literal $\text{not } c_i$ where $c_i \in A$, and then removing any defaultly negated literals in the remaining rules. An answer set $A$ satisfies a ground constraint $\leftarrow b_1, ..., b_n, \text{not } c_1,... \text{not } c_m.$ if it is not the case that $\{b_1, ..., b_n\} \in A$ and $A \cap \{c_1, ..., c_m\} = \emptyset$.